\documentclass[10pt,twocolumn,letter]{article}
\usepackage{times,leading,amsmath,amsthm,amssymb,xcolor,multirow,booktabs,graphicx,url}
\usepackage[text={7in,9in},centering]{geometry}
\leading{12pt}
\usepackage[textsize=scriptsize,textwidth=1.3cm]{todonotes}
\newcommand\eps{\ensuremath{\varepsilon}}
\usepackage{xcolor}

\begin{document}
\title{That which we call \emph{private}\footnote{A previous version of this note accompanied a poster presented at the 2019 USENIX Security Symposium. This is the full version.}}
\author{
\'Ulfar Erlingsson \and
Ilya Mironov\thanks{Work done while at Google.} \and
Ananth Raghunathan \and
Shuang Song}
\date{\vspace{-0.1in}\textit{Google Research, Brain}}
\maketitle

\begin{abstract}
The guarantees of security and privacy defenses
are often strengthened by
relaxing the assumptions
made about attackers
or the context in which defenses are deployed.
Such relaxations can be
a highly worthwhile topic of exploration---even though they
typically entail
assuming a weaker, less powerful adversary---because
there may indeed be great
variability in
both attackers' powers
and their context.

However,
no weakening or contextual discounting of attackers' power
is assumed
for
what some have called ``relaxed definitions''
in the analysis of
differential-privacy guarantees.
Instead, the definitions so named
are the basis of refinements and more advanced analyses
of the worst-case implications
of attackers---without any change assumed in attackers' powers.

Because they
more precisely bound
the worst-case privacy loss,
these improved analyses can greatly strengthen the differential-privacy upper-bound guarantees---sometimes
lowering the differential-privacy epsilon
by orders-of-magnitude.
As such, to the casual eye, these analyses may appear to imply a reduced privacy loss.
This is a false perception: the privacy loss of any concrete mechanism
cannot change with the choice of a worst-case-loss upper-bound analysis technique.
Practitioners must be careful not to equate real-world privacy
with differential-privacy epsilon values,
at least not without full consideration of the context.
\end{abstract}

\section{Introduction}
\begin{quote}
  ``\textit{What's in a name? That which we call a rose\\
  By any other name would smell as sweet.}''\vspace{-0.15in}\begin{flushright}
  Romeo and Juliet (II, ii, 1--2)\end{flushright}
\end{quote}
\vspace{-0.05in}
Alongside greatly increased adoption of machine learning (ML),
its privacy aspects have seen increased
attention, both
offensively~\cite{shokri2017membership,song2017machine,carlini2018secret,yeometal}
and defensively~\cite{DP-DL,mohassel2017secureml,pate,tfprivacy}.
In this, the gold-standard definitions of differential
privacy (DP) have rightly played a key role~\cite{DMNS}.
In particular,
in industry,
DP techniques are being used to
enable training of high-utility models that 
preserve the privacy of training data~\cite{rappor,appledp,DKY17-Microsoft,tfprivacy}.

While DP provides 
rigorous privacy guarantees,
they take the form of analytic upper bounds
that hold equally true for worst-case,
artificial adversarially-crafted scenarios
as they do for real-world ML applications.
As a result,
these upper-bound DP guarantees can be very loose
(i.e., 
overly pessimistic)
and the actual privacy loss
in real-world applications
may be many orders-of-magnitude lower 
than what is indicated by DP guarantees;
this is 
especially true 
in the analysis of ML models
trained using 
DP stochastic gradient descent (DP-SGD)~\cite{DP-DL}.
In addition,
the same  model
may be subjected---without any change or retraining---to 
different DP analyses
that give different upper bounds,
making DP guarantees even harder
to understand.

In particular,
a casual reader of the study by
Jayaraman and Evans
in USENIX Security 2019
might conclude that
``relaxed definitions of differential privacy'' 
should be avoided, 
because they
``increase the measured privacy leakage''
in the empirical study of DP machine-learning models~\cite{je}.

In this note, we demonstrate that this study is consistent with a different interpretation. Namely, the ``relaxed definitions'' are strict improvements and provide orders-of-magnitude tighter guarantees without changing the real-world privacy loss. We also reproduce the results of Jayaraman and Evans~\cite{je}, extend the bounds on privacy leakage defined by Yeom et al.~\cite{yeometal}, and demonstrate empirically a tighter connection between privacy leakage and differential privacy bounds. 

\section{An Apparent Paradox?}

The privacy of ML models trained by 
iterative DP-SGD with clipping and Gaussian noise~\cite{DP-DL} can
be analyzed via \emph{na\"ive}~\cite{ODO} or \emph{advanced}~\cite{DRV} composition theorems.
They may also be analyzed via more sophisticated and refined definitions
such as zero-Concentrated DP (\emph{zCDP})~\cite{zCDP} and R\'enyi DP (\emph{RDP})~\cite{RDP}.

For a fixed \eps{} upper-bound DP guarantee,
Jayaraman and Evans 
train models that achieve the target \eps{} (irrespective of utility)
under each of the above definitions;
subsequently,
they measure the models' empirical privacy loss
as the success rate of 
a variant of the membership inference attack in~\cite{yeometal}.
They find that the empirical attacks have higher success probability
for models
associated with refined
definitions (zCDP and RDP).
This result feels like a paradox, as zCDP and RDP were developed
to provide \emph{stronger} privacy guarantees incorporating tighter
and more advanced techniques than prior works.

This apparent paradox is resolved as follows. 
A model trained with DP-SGD 
has some (unknown) fixed actual privacy guarantee as well as (known) fixed utility. 
Such an existing model's empirical privacy
cannot be changed
by its re-analysis under refined definitions (we can only aim to get tighter upper bounds) 
and this is consistent with the study's results~\cite{je}.
Conversely, simpler DP definitions have (orders-of-magnitude) greater gaps between the upper-bound $\eps$
and actual privacy loss whereas empirical measurements are tied only to the actual privacy loss.
By training models to a fixed upper-bound $\eps$ 
under different DP definitions,
the models' \emph{actual} privacy loss will vary wildly (as will their utility).
We should expect 
the lowest empirical attack success rate
(and the lowest utility)
for the simplest DP definitions' models which is borne out by the study~\cite{je}.

\section{A New Interpretation}
\label{sec:new-interpretation}

\setlength{\fboxsep}{1pt}
\renewcommand{\arraystretch}{1.1}
\newcommand{\mybox}[1]{\fcolorbox{white}{yellow}{#1}}
\begin{center}
\begin{tabular}{r|l|l|l|l}
  \toprule
  \multirow{2}{*}{
  $\boldsymbol{\varepsilon}$} & \multicolumn{4}{c}{(\textit{cross-entropy loss}, \textit{count of attack success})} \\ \cline{2-5}
  & \textbf{Na\"ive} & \textbf{Advanced} & \textbf{zCDP} &
  \textbf{RDP} \\
  \hline
  1 & (.93, 0) & (.93, 0) & (.88, 0) & \mybox{(\textbf{.51}, 122)} \\
  10 & (.90, 0) & (.87, 0) & \mybox{(\textbf{.47}, 157)} & (.09, 329) \\
  100 & \mybox{(\textbf{.48}, 152)} & \mybox{(\textbf{.53}, 138)} & (.08, 362) & (.00, 456) \\
  \bottomrule
\end{tabular}
\end{center}

Consider this subset of
Table 5 in Jayaraman and Evans~\cite{je}
where pairs denote
cross-entropy loss and count of attack successes at 5\% FPR respectively.

On the diagonal, in yellow,
the empirical privacy loss can be seen to be approximately the same
for models with the same accuracy,
as predicted by the above discussion.

Training loss is a proxy for the effective noise added during DP-SGD training;
therefore, we can
``rotate the table'' to fix the noise (or
loss) for each row 
and reformat the table to show
($\eps$ upper bound, count of attack successes at 5\% FPR).
\setlength{\fboxsep}{1pt}
\begin{center}
\begin{tabular}{r|l|l|l|l}
  \toprule
  \textbf{$\mathbf{\approx}$} &
  \multicolumn{4}{c}{(\textit{$\eps$ upper bound}, \textit{count of attack successes})} \\ \cline{2-5}
  \textbf{Loss} & \textbf{Na\"ive} & \textbf{Advanced} & \textbf{zCDP} &
  \textbf{RDP} \\
  \hline
  .93 & (1, 0) & (1, 0) & (.1, 0) & (.05, 0) \\
  .65 & (50, 16) & (50, 73)
  & (5, 45) & (.5, 27)
  \\
  .50 & \mybox{(100, 152)} & \mybox{(100, 138)} & \mybox{(10, 157)} & \mybox{(1, 122)} \\
  \bottomrule
\end{tabular}
\end{center}

From this reformatted table,
it is obvious that---as expected---the more
advanced DP analysis 
provides orders-of-magnitude tighter upper-bound guarantees,
without changing the empirical privacy loss. 

\section{Reproducing Results}
We reproduce a subset of the results of Jayaraman and Evans~\cite{je}. In our first set of results, we picked $\eps$-value targets for training a Logistic Regression classifier for the Purchase-100 dataset and evaluated two aspects: (1) the resulting training accuracy relative to the baseline non-private training, and (2) the success probability of the membership inference attack from Yeom et al.~\cite{yeometal} comparing cross-entropy loss of the input to the average loss over the training set to classify the input as train or test. We present the results of (2) at a 5\% false positive rate, as in~\cite{je}.

\setlength{\fboxsep}{1pt}
\begin{center}
\begin{tabular}{r|l|l|l}
  \toprule
    \multirow{2}{*}{
  $\boldsymbol{\varepsilon}$} & \multicolumn{3}{c}{(\textit{training accuracy rel.~to baseline}, \textit{attack TPR})} \\ \cline{2-4}
  & \textbf{Na\"ive} & \textbf{Advanced} &
  \textbf{RDP} \\
  \hline
  1 & (2\%, 5.1\%) & (4\%, 5.2\%) & \mybox{(\textbf{14\%}, 5.6\%)} \\
  10 & (3\%, 5.1\%) & \mybox{(\textbf{12\%}, 5.6\%)} & (71\%, 7.8\%) \\
  100 & \mybox{(\textbf{11\%}, 5.9\%)} & (34\%, 6.1\%) & (92\%, 8.1\%)\\
  \bottomrule
\end{tabular}
\end{center}

We note a couple of changes. First, we do not report values where $\eps$ was targeted under a zCDP analysis. In the stochastic gradient descent setting of training models with Gaussian noise, zCDP and RDP definitions are identical and the results reported in~\cite{je} are simply weaker bounds that ignore how sampling affects the privacy bounds.

Second, we present much higher success rates for counts of attack success because we report this in a slightly different manner. Jayaraman and Evans report the number of consistently recovered members across five trials at 5\% false positive rate for each trial. We simply report the average number of recovered members across these trials. While the former gives an intuitive look into the efficacy of these attacks, the latter enables us to translate the \emph{true} TPR and FPR rates into lower-bound estimates for $\eps$-DP values of the trained model (see Section~\ref{sec:eps-bounds-attacks} for more information).

Similarly, in Table~\ref{tab:repro} we reproduce the results that go into the reformatted table from Section~\ref{sec:new-interpretation} by keeping the noise that goes into training as the constant across which both analysis upper bounds and attacks are compared. Once again, we conclude that advanced DP analyses provide significant improvements in the upper-bound guarantees while the empirical privacy loss, as measured by attacks, remains the same.

\begin{table*}[!htb]
\caption{Reproduced and re-oriented results on training models with noise.}
\label{tab:repro}
\setlength\tabcolsep{5pt} %
\centering
\begin{tabular*}{0.57\textwidth}{r|c|c|c|c|c}
  \toprule
    \multirow{2}{*}{\textbf{Noise mult.}} &
    \multirow{2}{*}{\textbf{Accuracy}} &
    \multirow{2}{*}{\textbf{Attack TPR}} &
    \multicolumn{3}{c}{\textit{$\eps$ at $\delta=10^{-5}$}} \\ \cline{4-6}
  & & & \textbf{Na\"ive} & \textbf{Advanced} &
  \textbf{RDP} \\
  \hline
  136.67$\times$ & 1.9\% & 5\% & 4.07 & 0.31 & 0.05\\
  13.76$\times$ & 7.0\% & 5.3\% & 51.2 & 4.23 & 0.5\\
  6.99$\times$ & 14\% & 5.6\% & 125 & 12.3 & 1\\
  1.06$\times$ & 71\% & 7.8\% & 8421 & 37722 & 10\\
  0.44$\times$ & 92\% & 8.9\% & 44770 & 3.5$\times 10^8$ & 100\\
  \bottomrule
\end{tabular*}
\end{table*}

\section{Epsilon Bounds from Attacks}
\label{sec:eps-bounds-attacks}

A model trained with $(\eps, \delta)$-DP guarantee mitigates membership inference attacks. Yeom et al.~\cite{yeometal} showed that the advantage in a membership inference game (defined equivalently as the difference between the TPR and FPR of the attack) can be bounded above as a function  of $\eps$. We get a tighter estimate (as a function of $\eps$ and $\delta$) below.

Equipped with this tighter bound, we can translate a membership inference attack with some success advantage to a lower bound on the \emph{actual} $\eps$ of the trained model. In other words, \emph{if} the model had an $\eps$ guarantee lower than this lower bound, the membership inference adversary cannot have such a high success probability. This allows us to consider a uniform picture of such an evaluation on differentially privately trained models---the best analysis provide the tightest upper bounds on the actual $\eps$, while the best attacks translate to the tightest lower bounds on the actual $\eps$. We can then conclude that the actual privacy guarantee of the trained model lies somewhere in between. 

\subsection{A tighter upper bound}

We strengthen the bound in~\cite[Theorem 1]{yeometal} by starting with the following proposition from Hall et al.~\cite[Proposition 4]{halldp} paraphrased.

\newtheorem{prop}{Proposition}

\begin{prop}
\label{prop:1}
Let $Y \sim \mathcal{M}(D)$ be the output of a mechanism over a dataset satisfying $(\eps,\delta)$-DP. Any test of the hypothesis that a particular $x \in D$ on input $Y$ with a false positive rate $\alpha$ has a true positive rate bounded above by $e^\eps \cdot \alpha + \delta$.
\end{prop}

Applying this, we can derive the following proposition in a straightforward manner.
\begin{prop} \label{prop:2}
The membership inference attack advantage of an adversary against a model trained with $(\eps,\delta)$-DP is bounded above by $1-e^{-\eps} + \delta \cdot e^{-\eps}$.
\end{prop}
\begin{proof}
Following the lines of the proof in~\cite[Theorem 1]{yeometal}, the membership advantage is equal to $\mathrm{TPR} - \mathrm{FPR}$. Let $\alpha$ denote the FPR and $(1-\beta)$ denote the TPR.\footnote{We use $\alpha$ and $\beta$ following conventional notation to denote significance level and Type II errors (or false negatives) respectively.} Proposition~\ref{prop:1} implies that $1-\beta \leq e^\eps \cdot \alpha + \delta \Longrightarrow \alpha \geq e^{-\eps}(1-\beta) - \delta \cdot e^{-\eps}$.

Therefore, the membership advantage, $(1-\beta)-\alpha$ is less than or equal to $ (1-\beta) - e^{-\eps}(1-\beta) + \delta \cdot e^{-\eps}$, which is equal to $ (1-e^{-\eps})(1-\beta) + \delta \cdot e^{-\eps}$. The proposition follows by noting that $\beta \geq 0$ and $e^{-\eps} \leq 1$.
\end{proof}

\subsection{Translating bounds}
We translate the stronger upper bound on attack advantage to a lower bound on the actual $\eps$ of the trained model in a straightforward manner by rearranging terms of Proposition~\ref{prop:2}.

\begin{prop}
\label{prop:eps-lower-bound}
For a given value of $\delta$, a noisily trained model admitting a membership inference attack, as defined by Yeom et al.~\cite{yeometal} with advantage $\gamma$ can only be $(\eps,\delta)$-DP for \begin{equation}\label{eq:eps-lower-bound}
\eps \geq \log\left(\frac{1-\delta}{1-\gamma}\right).\end{equation}
\end{prop}

We can reproduce a subset of Table~\ref{tab:repro} with attack-derived lower bounds noting that the actual privacy of the model lies somewhere between the best upper-bound estimate and this lower bound. (We also note that the lower bounds are loose as they are computed at FPR of 5\%, which does not necessarily translate to the best possible membership inference advantage.)
\begin{table}[!htpb]
\begin{center}
\begin{tabular}{r|c|c|c|c}
  \toprule
    \multirow{2}{*}{\textbf{Noise}} &
    \multirow{2}{*}{\textbf{$\boldsymbol{\eps}$-lower bound}} & \multicolumn{3}{c}{
    \textbf{$\boldsymbol{\eps}$-upper bound}} \\ \cline{3-5}
  &  & \textbf{RDP} & \textbf{Advanced} &
  \textbf{Na\"ive} \\
  \hline
  136.67$\times$ & 0 & 0.05 & 0.31 & 4.07
  \\
  13.76$\times$ & 0.003 & 0.5 & 4.23 & 51.2
  \\
  6.99$\times$ & 0.006 & 1 & 12.3 & 125
  \\
  1.06$\times$ & 0.04 & 10 & 37722 & 8421
  \\
  0.44$\times$ & 0.05 & 100 & 3.5$\times 10^8$ & 44770\\
  \bottomrule
\end{tabular}
\end{center}
\caption{Translated $\eps$-lower bounds and their corresponding upper bounds for various noise scales.}
\label{tab:lowerbounds}
\end{table}
\begin{figure}[!htpb]
\includegraphics[width=\columnwidth]{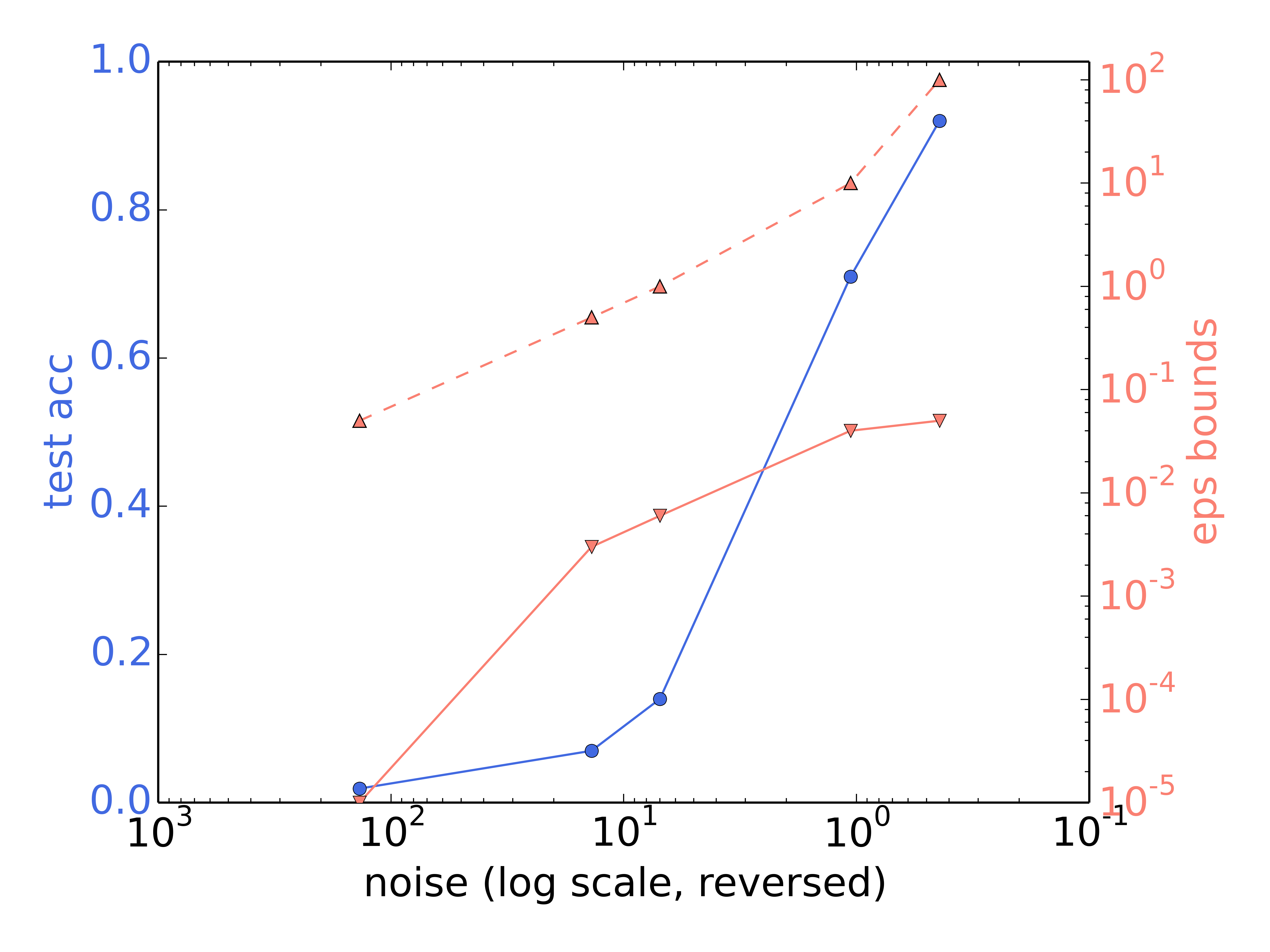}
\caption{The values from Table~\ref{tab:lowerbounds} plotted on a log-log graph that visualizes the gap between the lower and upper bounds.}
\end{figure}

\subsection{Region of privacy}

We combine the lower bound from Proposition~\ref{prop:eps-lower-bound} on the \emph{best} possible membership inference advantage (by choosing a threshold that maximizes $\mathrm{TPR}-\mathrm{FPR}$) and the upper bound from the \emph{best} possible analysis (RDP) to produce the graph below. The graph plots both bounds across differentially-privately trained models with noise varying over several orders of magnitude. The region between the two curves constitute the ``region of privacy'' where the actual privacy of the model lies. By improving the privacy analysis of models and getting stronger attacks which imply better lower bounds, we can narrow this region and get a more accurate estimate of how private the model is.

\begin{figure}[!htpb]
\includegraphics[width=\columnwidth]{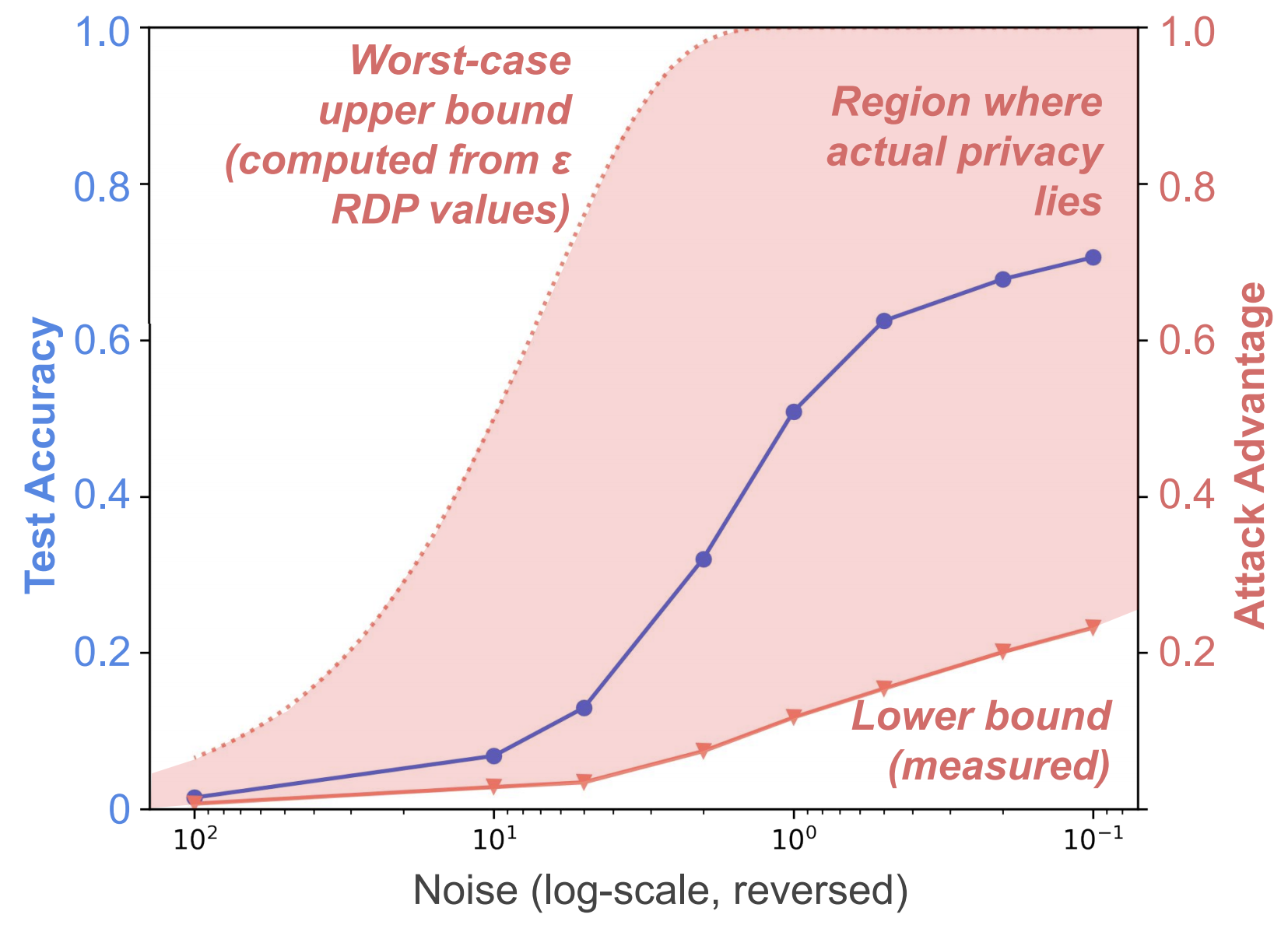}
\caption{Graph showing the region where the privacy of the model lies when trained with different noise levels.}
\end{figure}

\subsection{Discussion}
We conclude by noting important subtleties related to deriving epsilon bounds from attacks. Throughout this discussion, we make an implicit assumption about the training and test data. An attacker's advantage as defined in Yeom et al.~\cite{yeometal} being zero in an ideal scenario requires that the training and test distributions themselves be indistinguishable. Trivially, even a model trained with very good privacy properties might still yield a large attack advantage if the distributions are sufficiently different. In fact, one might be able to artificially construct pathological examples of good machine learning models trained with good accuracy/privacy tradeoffs that trivially fail these tests because, say, all train images start with a 0 pixel, and all test images start with a 1 pixel.

These tests shed a more accurate light on the privacy of the model if we simultaneously endeavor to make sure that the train and test distributions are as close as possible. This is closely related to similar work in the broader ML community exploring the connection between model generalization and artefacts in the train/test sets such as in the case of MNIST~\cite{coldcase} and CIFAR~\cite{recht2018cifar}.

A second subtlety we do not address in this work is the choice of threshold when applying the attacks from Yeom et al.~\cite{yeometal}. One straightforward way to choose this would be to consider average loss, as suggested in the original paper, but we could strengthen the attacks by considering the \emph{best} possible threshold that yields the largest attack advantage. While this sort of arbitrary auxiliary information (albeit minimal) would break the rigorous nature of the analysis of $\eps$-lower bounds, we still deem the resulting lower bound useful in practice for two reasons---the present upper bound analyses are loose, and the attacks mounted are fairly simplistic and therefore fairly loose lower bounds themselves. Strengthening lower bounds a little bit with arbitrary auxiliary information will not lead to inaccurate results.

\bibliographystyle{plain}
\bibliography{poster-abstract}
\end{document}